\documentclass{article}

\usepackage{microtype}
\usepackage{graphicx}
\usepackage{booktabs} 
\usepackage{subcaption}

\usepackage{xr-hyper}
\usepackage{hyperref}
\usepackage[sans]{dsfont}

\usepackage[utf8]{inputenc}
\usepackage{amsthm}
\usepackage{amsmath}
\usepackage{amssymb}  
\usepackage{xcolor}
\usepackage{shuffle}   
\usepackage{graphicx}
\usepackage{multirow}
\usepackage{makecell}
\usepackage{arydshln}
\usepackage{stmaryrd}

\usepackage{bbm}
\usepackage{wrapfig}

\usepackage[utf8]{inputenc} 
\usepackage[T1]{fontenc}    
\usepackage{url}            
\usepackage{booktabs}       
\usepackage{amsfonts}       
\usepackage{nicefrac}       
\usepackage{microtype}      
\usepackage{lipsum}
\usepackage{forest}
\usepackage{authblk}
\usepackage{cleveref}
\usepackage{mathtools}

\newtheorem{theorem}{Theorem}[section]

\newtheorem{lemma}{Lemma}[section]

\newtheorem{definition}{Definition}

\Crefformat{figure}{#2Fig.~#1#3}
\Crefmultiformat{figure}{Figs.~#2#1#3}{ and~#2#1#3}{, #2#1#3}{ and~#2#1#3}
\Crefformat{section}{#2Sec.~#1#3}
\Crefformat{theorem}{#2Thm.~#1#3}
\Crefformat{definition}{#2Def.~#1#3}
\Crefformat{algorithm}{#2Alg.~#1#3}
\usepackage[acronym,smallcaps]{glossaries}

\glsdisablehyper
\newacronym{VOSF}{VOSF}{\textit{Variational Orthogonal Signature Features}}
\newacronym{SigGPDE}{SigGPDE}{\textit{?}}

\newtheorem{innercustomgeneric}{\customgenericname}
\providecommand{\customgenericname}{}
\newcommand{\newcustomtheorem}[2]{%
  \newenvironment{#1}[1]
  {%
   \renewcommand\customgenericname{#2}%
  \renewcommand\theinnercustomgeneric{##1}%
   \innercustomgeneric
  }
  {\endinnercustomgeneric}
}

\newcustomtheorem{customthm}{Theorem}

\usepackage[accepted]{icml2021}

\icmltitlerunning{SigGPDE: Scaling Sparse Gaussian Processes on Sequential Data}

\makeatletter
\newcommand*{\addFileDependency}[1]{
  \typeout{(#1)}
  \@addtofilelist{#1}
  \IfFileExists{#1}{}{\typeout{No file #1.}}
}
\makeatother

\begin{document}

\twocolumn[

\icmltitle{SigGPDE: Scaling Sparse Gaussian Processes on Sequential Data}

\begin{icmlauthorlist}
\icmlauthor{Maud Lemercier}{war}
\icmlauthor{Cristopher Salvi}{ox}
\icmlauthor{Thomas Cass}{imp}
\icmlauthor{Edwin V. Bonilla}{csiro}
\icmlauthor{Theodoros Damoulas}{war}
\icmlauthor{Terry Lyons}{ox}

\end{icmlauthorlist}

\icmlaffiliation{war}{University of Warwick and Alan Turing Institute}
\icmlaffiliation{imp}{Imperial College London and Alan Turing Institute}
\icmlaffiliation{ox}{University of Oxford and Alan Turing Institute}
\icmlaffiliation{csiro}{CSIRO’s Data61 and The University of Sydney}

\icmlcorrespondingauthor{Maud Lemercier}{maud.lemercier@warwick.ac.uk}

\icmlkeywords{Machine Learning, ICML}

\vskip 0.3in
]



\printAffiliationsAndNotice{}  


\begin{abstract}
   Making predictions and quantifying their uncertainty when the input data is sequential is a fundamental learning challenge, recently attracting increasing attention. We develop SigGPDE, a new scalable sparse variational inference framework for Gaussian Processes (GPs) on sequential data. Our contribution is twofold. First, we construct inducing variables underpinning the sparse approximation so that the resulting evidence lower bound (ELBO) does not require any matrix inversion. Second, we show that the gradients of the GP signature kernel are solutions of a hyperbolic partial differential equation (PDE). This theoretical insight allows us to build an efficient back-propagation algorithm to optimize the ELBO. We showcase the significant computational gains of SigGPDE compared to existing methods, while achieving state-of-the-art performance for classification tasks on large datasets of up to $1$ million multivariate time series.
\end{abstract}

\section{Introduction}\label{sec:Introduction}

Gaussian process (GP) models provide a sound mathematical framework for supervised learning that allows the incorporation of prior assumptions and provides uncertainty estimates when modelling unknown functions \cite{rasmussen-williams-book}. This is usually achieved by specifying a GP prior over functions with a suitable covariance (or kernel) along with a conditional likelihood. With this, the problem boils down to that of estimating the posterior over the function (values) given the observed data. 

However, this posterior distribution is often analytically intractable and, even when the conditional likelihood is a Gaussian, GP models scale poorly on the number of observations $N$, with na\"{i}ve approaches having a time complexity $\mathcal{O}(N^3)$. From a wide range of  approximate techniques to scale inference in GP models to large datasets,  ``sparse'' methods based on variational inference (VI) have emerged as one of the dominant approaches \cite{titsias2009variational}. They consist in defining a family of approximate posteriors through $M$ \textit{inducing variables}, and selecting the distribution in this family that minimizes the Kullback-Leibler (KL) divergence between the approximation and the true posterior. This is achieved by maximizing the so-called evidence lower bound (ELBO). When the likelihood factorizes over datapoints, training can be done in minibatches of size $\tilde{N}$ resulting in a per-iteration computational cost $\mathcal{O}(\tilde{N}M^2+M^3)$, where the $\mathcal{O}(M^3)$ cost is due to the inversion of the covariance matrix of the $M$ inducing variables. This yields significant computational savings when $M\ll N$.

In the seminal work of \citet{titsias2009variational} the inducing variables correspond to evaluations of the GP at $M$ pseudo input locations, which typically results in a dense covariance matrix to invert. Subsequently, other ways of constructing inducing variables have been introduced in order to mitigate the $\mathcal{O}(M^3)$ cost \cite{hensman2017variational,burt2020variational}. The core idea consists in defining (almost) independent inducing variables, such that their covariance matrix is (almost) diagonal. These inducing variables correspond to projections of the GP on basis functions, such that the covariance matrix is a Gramian matrix with respect to some inner-product. Orthogonal basis functions yield diagonal Gramian matrices, hence these methods are often referred to as \textit{variational orthogonal features} (VOFs) . 
However existing VOF methods are limited to stationary kernels on   $\mathcal{X}\subset\mathbb{R}^d$ ($d \in \mathbb{N}$). 

In this work we are interested in generalizing the VOF paradigm to the case where the input space $\mathcal{X}$ is a set of \textit{sequences} of vectors in $\mathbb{R}^d$. One may be tempted to naively concatenate each vector in a sequence of length $\ell$ to form a flat vector in $\mathbb{R}^{\ell d}$. However in this case existing VOF methods cannot be directly applied because they are limited to low dimensional vectors, with $d \leq 8$ \cite{dutordoir2020sparse}. Thus, one needs kernel functions specifically designed for sequential data. The \textit{signature kernel} \cite{cass2020computing} is a natural choice that has recently emerged as a leading machine learning tool for learning on sequential data. In particular, \citet{toth2020bayesian} have proposed GPSig, a GP inference framework leveraging an approximation of this covariance function \cite{kiraly2019kernels} and achieving state-of-the-art performance on time series classification tasks. 
Nevertheless, as in standard sparse variational approaches to GPs, the inducing inputs they chose (so called \emph{inducing tensors}) are additional variational parameters to optimize, and the resulting covariance matrix is dense. 

Here we develop SigGPDE, a new scalable sparse variational inference framework for GP models on sequential data. After a brief recap on the general principles of variational inference (\Cref{sec:Background}) we identify a set of VOFs \emph{naturally} associated with the signature kernel. These inducing variables do not depend on any variational parameter as they are defined as projections of GP-samples onto an orthogonal basis for the RKHS associated to the signature kernel (\Cref{sec:VOS}). As a result, unlike the methods developed in \citet{toth2020bayesian}, in SigGPDE the optimization of the ELBO \textit{does not require any matrix inversion}. Subsequently, we show that the gradients of the signature kernel are solutions of a \emph{hyperbolic partial differential equation} (PDE). This theoretical insight allows us to build an efficient back-propagation algorithm to optimize the ELBO (\Cref{sec:backprop_PDE}). Our experimental evaluation shows that SigGPDE is considerably faster than GPSig, whilst retaining similar predictive performances on datasets of up to $1$ million multivariate time series (\Cref{sec:experiments}).

\section{Background}\label{sec:Background}

We begin with a general summary of variational inference for GPs. In this section, it is assumed that the input space is $\mathcal{X} \subset \mathbb{R}^d$. Standard models with zero-mean GP priors and iid conditional likelihoods can be written as follows
\begin{equation}
    f \sim \mathcal{GP}(0,k(\cdot,\cdot))
     \quad p(\mathbf{y}|\mathbf{f}) = \prod_{i=1}^{N}p(y_i|f(x_i)),
      \label{eq:GP}
\end{equation}
where $k(\cdot,\cdot)$ is the covariance function. 
The general setting for sparse GPs consists in specifying a collection of $M$ variables as well as a joint distribution with variational parameters $\mathbf{m}$ (mean vector) and $\Sigma$ (covariance matrix)
\begin{align}
    \mathbf{u}=\{u_m\}_{m=1}^{M}, \quad q(\mathbf{u})=\mathcal{N}(\mathbf{m},\Sigma).
\end{align} 
These variables induce a family of approximate posteriors that are GPs with finite dimensional marginal densities of the form $q(\mathbf{f},\mathbf{u})= p(\mathbf{f}|\mathbf{u})q(\mathbf{u})$. Considering any input $x \in \mathcal{X}$, the mean and covariance functions of these GPs are
\begin{align}
    \mu_q(x) &= C_{f_x\mathbf{u}}C_{\mathbf{u}\mathbf{u}}^{-1}\mathbf{m}\\ 
    k_q(x,y) &= k(x,y)-C_{f_x \mathbf{u}}C_{\mathbf{u}\mathbf{u}}^{-1}(C_{\mathbf{u}\mathbf{u}}-\Sigma)C_{\mathbf{u}\mathbf{u}}^{-1}C_{\mathbf{u}f_{y}},\notag
\end{align}
where the vector $C_{f_x\mathbf{u}}$ and the matrix $C_{\mathbf{u}\mathbf{u}}$ are defined as 
\begin{equation}
[C_{f_x\mathbf{u}}]_m = \mathbb{E}[u_mf(x)], \ \quad [C_{\mathbf{u}\mathbf{u}}]_{m,m'} = \mathbb{E}[u_mu_{m'}]   
\end{equation}
Provided the inducing variables $\mathbf{u}$ are deterministic conditioned on $f$, one has the following lower bound (ELBO) on the marginal log likelihood \cite{matthews2017scalable}
\begin{align}\label{eq:ELBO}
    \hspace{-0.05cm}\log p(\mathbf{y}) \geq \mathbb{E}_{q(\mathbf{f})}[\log p(\mathbf{y}|\mathbf{f})] -KL[q(\mathbf{u})||p(\mathbf{u})],
\end{align}%
where $p(\mathbf{u})=\mathcal{N}(0_M,C_{\mathbf{u}\mathbf{u}})$. Maximizing the right-hand-side of \cref{eq:ELBO} is equivalent to minimizing the KL divergence between $q(f)$ and the true posterior distribution. 

The original variational inference framework outlined in \citet{titsias2009variational} consists in setting $u_m=f(z_m)$ where $z_m \in \mathcal{X}$ is a pseudo input living in the same space as $x$ that may either be fixed or optimized. The per-iteration cost of optimizing the ELBO is $\mathcal{O}(\tilde{N}M^2+M^3)$, where $\tilde{N}$ is the minibatch size and $M^3$ is the cost of computing $C_{\mathbf{uu}}^{-1}$ via a Cholesky decomposition.

Recently, a considerable effort has been devoted to the construction of inducing variables $\mathbf{u}$ which yield a structured covariance matrix  $C_{\mathbf{u}\mathbf{u}}$ whose inversion has a reduced computational complexity \citep{hensman2017variational}. This line of work is often referred to as \textit{inter-domain} sparse GPs, owing to the fact that the pseudo inputs are not constrained to live in $\mathcal{X}$ as before. In particular, \citet{burt2020variational,dutordoir2020sparse} have shown that provided one can find an orthogonal basis of functions for the RKHS associated with the kernel $k(\cdot,\cdot)$, it is possible to define the inducing variables as projections of the GP samples onto this basis. This construction yields a diagonal covariance matrix $C_{\mathbf{u}\mathbf{u}}$.

\section{Variational Inference with Orthogonal Signature Features} \label{sec:VOS}

Here we present our first contribution, namely the use of \emph{orthogonal signature features} as inducing variables for GPs on sequential data. We begin with a summary of the theoretical background needed to define GPs endowed with the signature kernel. In this section $\mathcal{X}$ is no longer a subspace of $\mathbb{R}^d$ but will be defined as a space of \emph{paths} hereafter.

\subsection{The signature}\label{ssec:sig-features}

Consider a time series $\mathbf{x}$ as a collection of points $x_i \in \mathbb{R}^{d-1}$ with corresponding time-stamps $t_i \in \mathbb{R}$ such that 
\begin{equation}
    \mathbf{x} = ((t_0,x_0), (t_1, x_1), ..., (t_n, x_n))
\end{equation}
with $0=t_0 < ... < t_n=T$. Let $X : [0,T] \to \mathbb{R}^d$ be the piecewise linear interpolation of the data such that $X_{t_i} = (t_i, x_i)$. We denote by $\mathcal{X}$ the set of all continuous piecewise linear paths defined over the time interval $[0,T]$ and with values on $\mathbb{R}^d$. 

For any path $X \in \mathcal{X}$ and any $\alpha \in \{1,\ldots,d\}$, we will denote its $\alpha^{th}$ \emph{channel} by $X^{(\alpha)}$ so that at any time $t \in [0,T]$
\begin{equation}
    X_t = (X^{(1)}_t, \ldots, X^{(d)}_t),
\end{equation}
as depicted on \Cref{fig:path} with $d=3$.

The \textit{signature} $S:\mathcal{X} \to H$ is a \emph{feature map} defined for any path $X \in \mathcal{X}$ as the following infinite collection of statistics
\begin{align*}
    S(X) = \Big(1, &\left\{S(X)^{(\alpha_1)}\right\}_{\alpha_1=1}^d,\\
    &\left\{S(X)^{(\alpha_1,\alpha_2)}\right\}_{\alpha_1,\alpha_2=1}^d, \\
    &\left\{S(X)^{(\alpha_1,\alpha_2,\alpha_3)}\right\}_{\alpha_1,\alpha_2,\alpha_3=1}^d,\ldots \Big)
\end{align*}
where each term is a scalar equal to the iterated integral 
\begin{equation}\label{eqn:sig}
S(X)^{(\alpha_1,\ldots,\alpha_j)} =  \underset{0<s_1<\ldots<s_j<T}{\int \ldots \int} d X^{(\alpha_1)}_{s_1} \ldots dX^{(\alpha_j)}_{s_j}
\end{equation}
The \emph{feature space} $H$ associated to the signature is a Hilbert space defined as the direct sum of tensor powers of $\mathbb{R}^d$
\begin{equation}
    H = \bigoplus_{k=0}^\infty (\mathbb{R}^{d})^{\otimes k} = \mathbb{R} \oplus \mathbb{R}^{d}\oplus (\mathbb{R}^{d})^{\otimes 2} \oplus \ldots 
\end{equation}
where $\otimes$ denotes the outer product \cite{lyons1998differential, lyons2014rough}.

\paragraph{Interpretability of the signature features} 
An important aspect of sequential data is that the order of the observations is not commutative, in the sense that reordering the elements of a sequence can completely change its meaning. By definition the terms in the signature capture this fact. In effect, the $j$-fold iterated integral in \cref{eqn:sig} is defined as an integral over the simplex $0<s_1<\ldots<s_j<T$ which explicitly encodes the ordering of events happening across different channels $X^{(\alpha_1)},\ldots,X^{(\alpha_j)}$. 
This provides the signature features with a natural interpretability as highlighted several times in prior work \citep{arribas2018signature,moore2019using,lemercier2020distribution}.

\subsection{The signature kernel}

The \emph{signature kernel} $k:\mathcal{X}\times\mathcal{X} \to \mathbb{R}$ is a reproducing kernel associated to the signature feature map and defined for any pair of paths $X,Y \in \mathcal{X}$ as the following inner product
\begin{align}\label{eq:kernel}
    k(X,Y) = \left\langle S(X), S(Y) \right\rangle_H. 
\end{align} 
From the structure of $H$ and the properties of the signature it turns out that the signature kernel can be decomposed according to the expansion  \cite{cass2020computing}
\begin{align}\label{eq:k}
    k(X,Y)=\sum_{j=0}^\infty \sum_{|\boldsymbol\alpha|=j}S(X)^{\boldsymbol\alpha}S(Y)^{\boldsymbol \alpha},
\end{align}
where the inner summation is over the set of multi-indices
\begin{align}
    \left\{\boldsymbol \alpha = (\alpha_1,\ldots,\alpha_j) : \alpha_1,\ldots,\alpha_j \in \{1,\ldots,d\}\right\}
\end{align}
In their recent article, \citet{cass2020computing} provide a \emph{kernel trick} for the signature kernel by proving the relation
\begin{equation}\label{eq:pde_kernel}
    k(X,Y)=U(T,T)
\end{equation}
where the function of two variables $U : [0,T] \times [0,T] \to \mathbb{R}$ is the solution of the following hyperbolic PDE
\begin{align}\label{eq:sig_PDE}
    \frac{\partial^2 U}{\partial s \partial t} = (\dot X_s^T\dot Y_t) U
\end{align}
with boundary conditions $U(0,\cdot)=1$ and $U(\cdot,0)=1$.
This kernel trick is explained with simple arguments in the proof of \citet[Thm. 2.5]{cass2020computing}. The sketch of the proof goes as follows: one first shows that the inner-product in \cref{eq:kernel} satisfies a double integral equation which comes from the fact that
the signature itself solves an integral equation. Then one uses the fundamental theorem of calculus to differentiate with respect to the two time variables to obtain the PDE.

\begin{figure*}
    \centering
    \hspace{-1cm}
    \begin{subfigure}[b]{0.3\textwidth}
        \centering
        \includegraphics[trim={2.6cm 0.5cm 0.5cm 0},clip,scale=0.3]{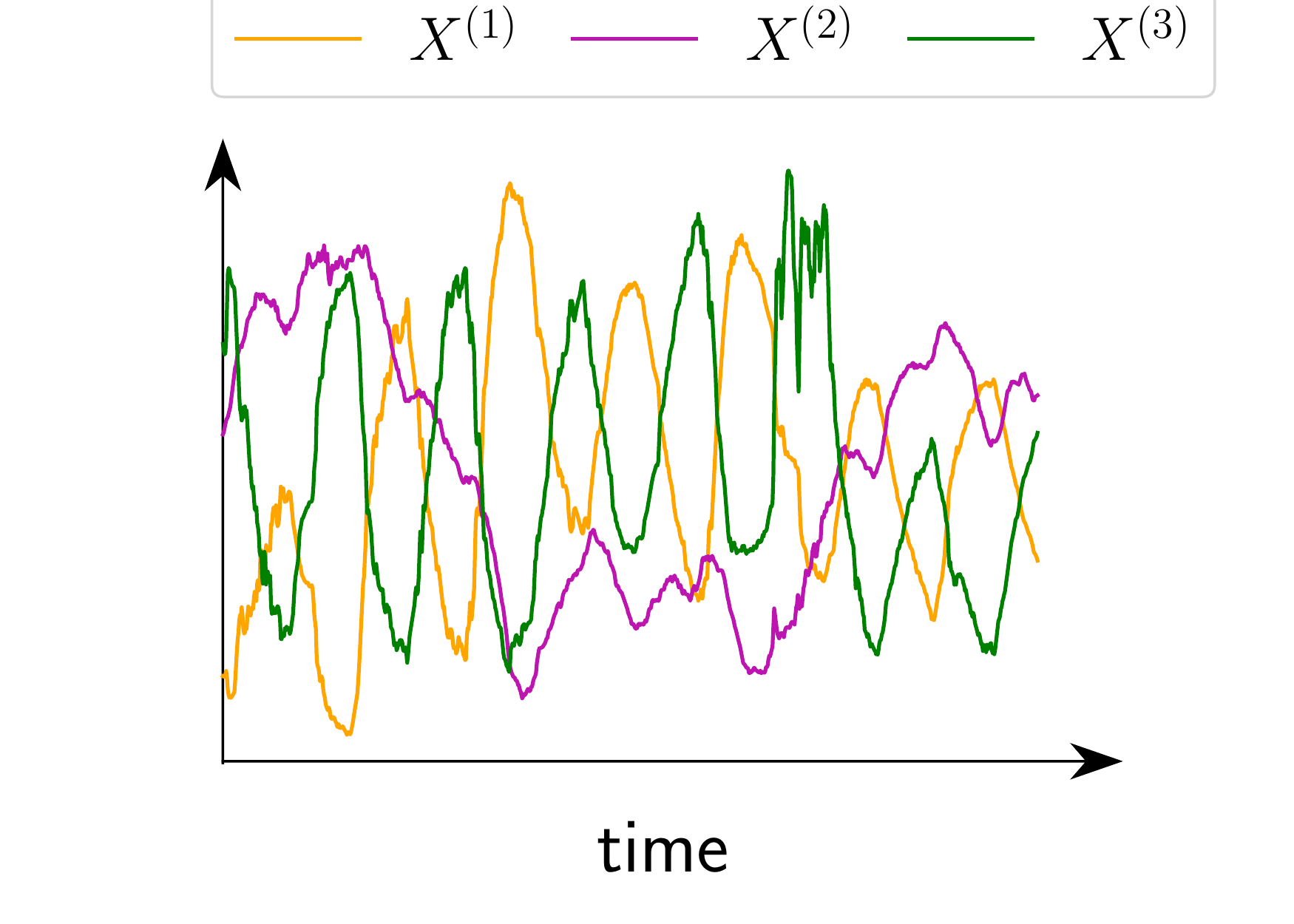}
        \caption{}
        \label{fig:path}
    \end{subfigure}%
    \hspace{-0.3cm}
    \begin{subfigure}[b]{0.7\textwidth}
        \centering
        \includegraphics[trim={6.2cm 15.075cm 5cm 1cm},clip,scale=0.9]{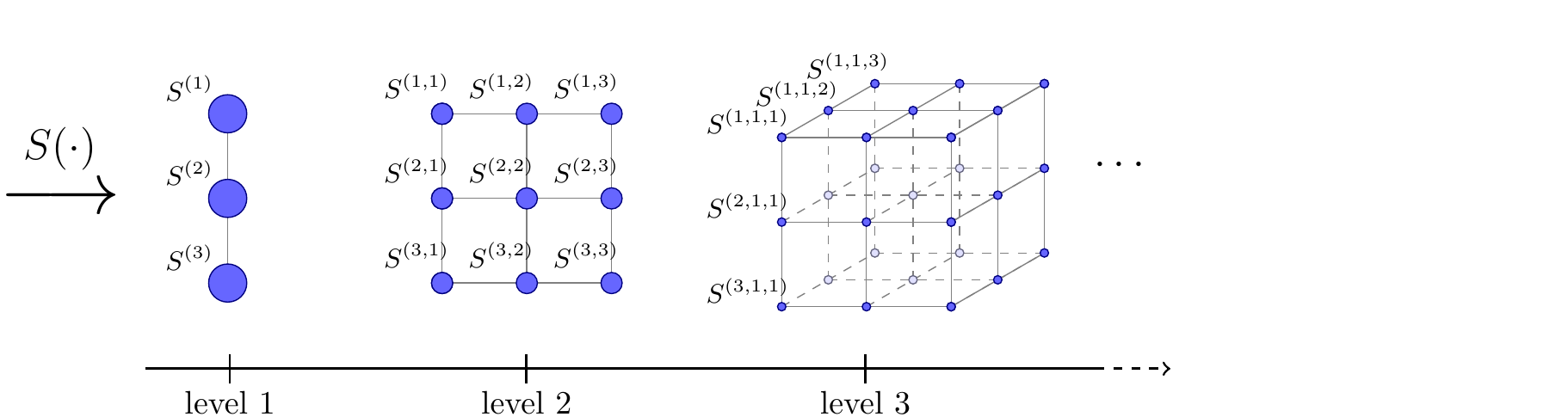}
        \caption{}
    \end{subfigure}
    \caption{Illustration of the first terms of the signature $S(X)$ for a 3-dimensional path $X$. Each blue circle corresponds to a signature feature $S(X)^{\boldsymbol \alpha}$ with $\boldsymbol \alpha=(\alpha_1,\ldots,\alpha_j)$. The size of the circle reflects the feature importance according to the property $|S(X)^{\boldsymbol \alpha})|=\mathcal{O}(1/|\boldsymbol \alpha|!)$. The first feature $S^{(0)}$ which is always equal to $1$ is omitted in this schematic.}
     \label{fig:factorial}
\end{figure*}

Next, we propose a simple parametrization of this kernel.

\subsection{Parametrization of the signature kernel}\label{ssec:parametrization}

In many real-world problems the
input path $X$ contains a large number $d$ of different channels, only some of which are relevant. For any coordinate $\alpha \in \{1,\ldots,d\}$ and time index $t \in [0,T]$ one can rescale each channel $X^{(\alpha)}$ by a scalar hyperparameter $\theta_\alpha$ yielding the rescaled path
\begin{align}\label{eq:rescale_path}
    X^{\boldsymbol{\theta}}_t := (\theta_1 X^{(1)}_t,\ldots,\theta_d X^{(d)}_t).
\end{align}
From \cref{eqn:sig} it is straightforward to see that the corresponding rescaled signature satisfies the following relation
\begin{align}
S_{\boldsymbol{\theta}}(X)^{(\alpha_1,\ldots,\alpha_j)} &:= S(X^{\boldsymbol{\theta}})^{(\alpha_1,\ldots,\alpha_j)}\\
&= \theta_{\alpha_1}\ldots \theta_{\alpha_j}S(X)^{(\alpha_1,\ldots,\alpha_j)}
\end{align}
for any $\alpha_1,\ldots,\alpha_j \in \{1,\ldots,d\}$. As a result, akin to an \textit{automatic relevance determination} (ARD) parametrization, the signature kernel of \cref{eq:k} can be reparametrized as
\begin{align}\label{eq:k_param1}
    k_{\boldsymbol{\theta}}(X,Y)=\sum_{j=0}^\infty \sum_{|\boldsymbol \alpha|=j}S_{\boldsymbol{\theta}}(X)^{\boldsymbol \alpha}S_{\boldsymbol{\theta}}(Y)^{\boldsymbol\alpha}.
\end{align}

\subsection{Variational Orthogonal Signature Features}
In the sequel we build on the results from the previous sections to define the orthogonal signature features underlying our sparse variational inference framework for GPs on sequential data. 

By \citet[Thm. 4.21]{steinwart2008support}, the \emph{reproducing kernel Hilbert space} (RKHS) $\mathcal{H}$ associated to the parametrized signature kernel $k_{\boldsymbol{\theta}}$ can be defined as
\begin{align}
    \mathcal{H} = \left\{g: X \mapsto \left\langle S_{\boldsymbol{\theta}}(X), \mathbf h \right\rangle_{H} \right\}, \quad \mathbf{h} \in H.
\end{align}

Besides, for any two functions $g_1,g_2 \in \mathcal{H}$ such that, 
\begin{align}
    g_1 &: X \mapsto \langle S_{\boldsymbol{\theta}}(X),\mathbf{h}_1 \rangle_{H}\\
    g_2 &: X \mapsto \langle S_{\boldsymbol{\theta}}(X),\mathbf{h}_2 \rangle_{H},
\end{align}
the inner product $\langle \cdot, \cdot \rangle_H$ induces the inner product on $\mathcal{H}$ 
\begin{equation}\label{eqn:f_RKHS}
    \langle g_1,g_2 \rangle_{\mathcal{H}} = \langle \mathbf{h}_1,\mathbf{h}_2\rangle_{H}.
\end{equation}
This result relies on the unicity of $\mathbf{h}_1$ and $\mathbf{h}_2$ in the decomposition of $g_1$ and $g_2$, which follows from \citet[Lemma 3.4]{diehl2019invariants} and \citet[Lemma 5]{xu2007refinable}.\footnote{By \citet[Lemma 3.4]{diehl2019invariants} $\overline{\mathrm{span}}S(\mathcal{X})=H$ where $S(\mathcal{X}):=\{S(X):~X\in\mathcal{X}\}$. Therefore by \citet[Lemma 5.]{xu2007refinable} $\forall g\in\mathcal{H}, ~\exists !\mathbf{h}\in H:g(\cdot)=\langle S(\cdot),\mathbf{h}\rangle_{H}$.} 

The key to our setup is that the set of signature features 
\begin{equation}
    \mathcal{S}^\perp = \left\{S_{\boldsymbol \theta}(\cdot)^{\boldsymbol \alpha} : X \mapsto S_{\boldsymbol\theta}(X)^{\boldsymbol \alpha}\right\}_{\boldsymbol\alpha = (\alpha_1,\ldots,\alpha_j)} 
\end{equation}
forms an orthonormal basis for the RKHS $\mathcal{H}$, i.e.
\begin{equation}\label{eq:orthogonality}
    \left\langle S_{\boldsymbol\theta}(\cdot)^{\boldsymbol \alpha}, S_{\boldsymbol\theta}(\cdot)^{\boldsymbol \alpha'} \right\rangle_{\mathcal{H}} = \delta_{\boldsymbol \alpha,\boldsymbol \alpha'} = \begin{cases}
            1, &         \text{if } \boldsymbol \alpha =\boldsymbol \alpha' ,\\
            0, &         \text{otherwise}.
    \end{cases}
\end{equation}
An important property of the orthonormal basis $\mathcal{S}^\perp$ is that its elements are naturally ordered. This ordering is due to the property that for any path $X\in\mathcal{X}$ the terms of the signature decay factorially \cite{lyons2007differential}
\begin{equation}\label{eq:factorial_decay}
    |S(X)^{\boldsymbol \alpha}|=\mathcal{O}\left(\frac{1}{|\boldsymbol\alpha|!}\right),
\end{equation}
as shown in \Cref{fig:factorial}. Hence to index the signature orthogonal features $S_{\boldsymbol \theta}^1, S_{\boldsymbol \theta}^2,\ldots,S_{\boldsymbol \theta}^m,\ldots$ we order first by increasing level $j$, and then by sorting the multi-indices $\boldsymbol{\alpha}$ within a level. From \cref{eq:orthogonality,eq:factorial_decay} we define our inducing variables as orthogonal projections\footnote{Although $f$ does not belong to $\mathcal{H}$ with probability $1$ \citep{kanagawa2018gaussian}, such projections are well defined because the space spanned by $\mathcal{S}^{\perp}$
is dense in the space of continuous functions on $\mathcal{X}$ and $f$ is continuous \citep[Thm. 1.]{toth2020bayesian}} of the GP onto the first $M$ elements of the orthonormal signature basis $\mathcal{S}^\perp$, that is
\begin{align}\label{eq:inducingvarproj}
    u_m=\langle f,S_{\boldsymbol{\theta}}^m\rangle_{\mathcal{H}}, \quad 1 \leq m \leq M.
\end{align}
With this choice of inducing variables we easily deduce the following covariances \cite{hensman2017variational}
\begin{align}
    \mathbb{E}[u_mf(X)]=S_{\boldsymbol{\theta}}^{m}(X) \mbox{~~~and~~~} \mathbb{E}[u_mu_{m'}]=\delta_{m,m'},
\end{align}
which implies that the covariance matrix $C_{\mathbf{uu}}$ is the identity. 

For any path $X \in \mathcal{X}$ we use the convenient vector notation
\begin{equation}
    \mathcal{S}_M(X) := [S_{\boldsymbol{\theta}}^1(X), \ldots, S_{\boldsymbol{\theta}}^M(X)] \in \mathbb{R}^M,
\end{equation}
to obtain the approximate posterior $\mathcal{GP}(\mu,\nu)$ with mean and covariance functions defined by the following equations
\begin{align}\label{eq:approximate_posterior}
    \mu(X)&=\mathcal{S}_M(X)^T\mathbf{m}\\
    \nu(X,Y)&=k_{\boldsymbol{\theta}}(X,Y)-\mathcal{S}_M(X)^T(I_{M}-\Sigma)\mathcal{S}_M(Y).\notag
\end{align}
We note that the signature and the signature kernel can be easily computed on real time series using existing python libraries \cite{esig, reizenstein2018iisignature}.

\section{Reverse-mode automatic differentiation for the signature kernel}\label{sec:backprop_PDE}

In order to optimize the ELBO with respect to the parameters $\boldsymbol{\theta}$  one needs to take derivatives of the signature kernel $k_{\boldsymbol{\theta}}$ of \cref{eq:approximate_posterior} with respect to each of its input paths.  Given that $k_{\boldsymbol{\theta}}$ solves the PDE (\ref{eq:sig_PDE}) it can be computed using appropriate PDE numerical solvers. Therefore, in theory the differentiation could be carried out by leveraging the automatic differentiation tools of modern deep learning libraries (Tensorflow, PyTorch etc.). However, backpropagating through the operations of the PDE solver can be inefficient. 

Here we show that the gradients of $k_{\boldsymbol{\theta}}$ can be computed efficiently \textit{without backpropagating through the operations of the PDE solver} as they are the solutions of a second PDE analogous to \cref{eq:sig_PDE}. The ability not to rely on automatic differentation allows for an efficient fitting of SigGPDE both in the terms of time complexity and memory cost. 

\begin{algorithm*}
\caption{Backpropagation for $k_{\boldsymbol{\theta}}(X,X)$ via PDE (\ref{eqn:deriv_PDE})}
\label{alg:PDE1}
\begin{algorithmic}[1]
\STATE {\bfseries Input:} Path $X$, localised impulses $\boldsymbol{\gamma} = \{\gamma_{i,j}\}$ fully determined by the time series $\mathbf x$.
\STATE\hspace{10pt}$\mathbf{u}_{0,:}=[1,0,\ldots,0], \ \ \mathbf{u}_{:,0}=[1,0,\ldots,0]$  \hfill\textit{// Boundary conditions for the augmented state}
\STATE{\hspace{10pt}\bfseries def} aug\_dynamics$\left(\left[U(s,t),U_{\boldsymbol\gamma}(s,t)\right],s,t\right)$: \hfill\textit{// Dynamics for the augmented state}
\STATE{\hspace{25pt}\bfseries return $\left[\dot{X}_s^T\dot{X}_tU(s,t),\dot{X}_s^T\dot{X}_tU_{\boldsymbol\gamma}(s,t)+\dot{\boldsymbol\gamma}_s\dot{X}_tU(s,t)\right]$}
\STATE\hspace{10pt}$\left[U(T,T),U_{\boldsymbol\gamma}(T,T)\right]=\mathrm{PDESolve}(\mathbf{u}_{0,:},\mathbf{u}_{:,0},\mathrm{aug\_dynamic},T,T)$

\STATE{\bfseries Output:} $2\cdot U_{\boldsymbol\gamma}(T,T)$ \hfill\textit{// Gradients of the kernel at the knots of $X$}
\end{algorithmic}
\end{algorithm*}
\subsection{Differentiating the signature kernel along the direction of a path}

Consider a time series $\mathbf{x}$ as a collection of points $x_i \in \mathbb{R}^d$ with corresponding time-stamps $s_i \in \mathbb{R}$ such that 
\begin{equation}
    \mathbf{x} = ((s_0, x_0), (s_1, x_1), ..., (s_\ell, x_\ell))
\end{equation}
with $s_0 < ... < s_\ell$. Every vector $x_i$ in the sequence can be written with respect to the canonical basis of $\mathbb{R}^d$ as
\begin{equation}
    x_i = \sum_{j=1}^d x_{i,j}\mathbf e_j
\end{equation}

Let $X: [0,T] \to \mathbb{R}^d$ be the piecewise linear interpolation of the data such that $X_{t_i} = (t_i, x_i)$. Similarly for a second time series $\mathbf y$ and resulting piecewise linear interpolation $Y$. Recall the definition of signature kernel as
\begin{align}
    k_{\boldsymbol{\theta}}(X,Y)=k(X^{\boldsymbol{\theta}},Y^{\boldsymbol{\theta}}),
\end{align}
where $X^{\boldsymbol{\theta}}$ and $Y^{\boldsymbol{\theta}}$ are the rescaled paths of \cref{eq:rescale_path}. 

By the chain rule one has that
\begin{align}
    \frac{\partial k_{\boldsymbol{\theta}}}{\partial\boldsymbol{\theta}}=\frac{\partial k}{\partial X^{\boldsymbol{\theta}}}\frac{\partial X^{\boldsymbol{\theta}}}{\partial \boldsymbol{\theta}}+\frac{\partial k}{\partial X^{\boldsymbol{\theta}}}\frac{\partial Y^{\boldsymbol{\theta}}}{\partial\boldsymbol{\theta}}
\end{align}

Hence, to formulate a backpropagation algorithm in a rigorous way compatible with the TensorFlow library used in this work, we need to give meaning to the following gradients
\begin{equation}\label{eqn:meaning}
    \left\{\frac{\partial}{\partial x_{i,j}}k(X,Y)\right\}_{i,j=1}^{\ell,d}
\end{equation}
The technical difficulty here consists in reconciling the continuous nature of the input path $X$ and the discrete nature of the locations $x_{i,j}$ where one wants to compute the gradients and given by the knots of the time series $\mathbf x$.

Next we introduce a collection of \textit{localised impulses} and define the concept of \emph{directional derivative of the signature kernel along a path} in order to make sense of the gradients in \cref{eqn:meaning}. These definitions will be followed by the main result of this section, namely that the directional derivative of $k$ solves another PDE similar to \cref{eq:sig_PDE} for the signature kernel, for which we derive an explicit solution via the \emph{technique of variation of parameters} (\Cref{thm:kernelvarparams}).

\begin{definition}\label{def:loc_imp}
For any $i=1,\ldots, \ell$ and any $j=1, \ldots, d$ define the localised impulse $\gamma_{i,j} : [0,T] \to \mathbb{R}^d$ as the solution of the following ordinary differential equation (ODE)
\begin{equation}
    \dot \gamma_{i,j} = \frac{1}{\ell} e_j \mathds{1}_{\{t\in[(i-1)/\ell, i/\ell)]\}}, \quad \gamma_{i,j}(0)=0
\end{equation}
\end{definition}

\begin{definition}
For any path $\gamma \in \mathcal{X}$ the directional derivative of the signature kernel $k$ along $\gamma$ is defined as 
\begin{align}\label{eqn:dir_derpaper}
    k_\gamma(X,Y) := \frac{\partial}{\partial \epsilon} k\Big(X+\epsilon \gamma, Y\Big)\Big|_{\epsilon = 0}
\end{align}
\end{definition}

Each gradient of the signature kernel $k_{\boldsymbol \theta}$ at the knot $x_{i,j}$ reported in \cref{eqn:meaning} can be identified with the directional derivative of $k_{\boldsymbol \theta}$ along the localised impulse $\gamma_{i,j}$ of \Cref{def:loc_imp}
\begin{equation}
    \frac{\partial}{\partial x_{i,j}}k(X,Y) :=  k_{\gamma_{i,j}}(X,Y)
\end{equation}

\subsection{A PDE for the gradients of the signature kernel}
\begin{algorithm*}
\caption{Backpropagation for $k_{\boldsymbol{\theta}}(X,X)$ via variation of parameters (\Cref{thm:kernelvarparams})}
\label{alg:PDE2}
\begin{algorithmic}[1]
\STATE {\bfseries Input:} Path $X$, localised impulses $\boldsymbol\gamma = \{\gamma_{i,j}\}$ fully determined by the time series $\mathbf x$
\STATE \hspace{10pt}$\mathbf{u}_{0,:}=[1,\ldots,1], \ \ \mathbf{u}_{:,0}=[1,\ldots,1]$ \hfill\textit{// Boundary conditions for the augmented state}
\STATE{\hspace{10pt}\bfseries def} aug\_dynamics$\left(\left[U(s,t),\tilde{U}(s,t)\right],s,t\right)$: \hfill\textit{// Dynamics for the augmented state}
\STATE{\hspace{25pt}\bfseries return $\left[\dot{X}_s^T\dot{X}_tU(s,t),\dot{X}_{T-s}^T\dot{X}_{T-t}\tilde{U}(s,t)\right]$}
\STATE\hspace{10pt}$[U,\tilde{U}]=\mathrm{PDESolve}(\mathbf{u}_{0,:},\mathbf{u}_{:,0},\mathrm{aug\_dynamic},T,T)$ \hfill\textit{// Keep the solutions at each $(s,t)$}
\STATE\hspace{10pt}$U_{\boldsymbol{\gamma}}=$ tf.sum($U \cdot \tilde{U} \cdot \boldsymbol\gamma X$)  \hfill\textit{// Simple final TensorFlow operations}
\STATE{\bfseries Output:}  $2\cdot U_{\boldsymbol\gamma}$ \hfill\textit{// Gradients of the kernel at the knots of $X$}
\end{algorithmic}
\end{algorithm*}

Recall that the signature kernel $k_\theta$ solves the following PDE
\begin{align}\label{eq:sig_PDE_}
    \frac{\partial^2 U}{\partial s \partial t} = (\dot X_s^T\dot Y_t) U
\end{align}
Integrating both sides with respect to $s$ and $t$ one obtains
\begin{equation}\label{eqn:sig_integral}
    U(s,t) = 1 + \int_{u=0}^s\int_{v=0}^t U(u,v) (\dot{X}_u^T\dot{Y}_v)dudv
\end{equation}
Let's denote by $U_\gamma : [0,T] \times [0,T] \to \mathbb{R}$ the directional derivative $k_\gamma$ evaluated at the restricted paths $X|_{[0,s]},Y|_{[0,t]}$
\begin{equation}\label{eq:U_der}
    U_\gamma(s,t) := k_\gamma(X|_{[0,s]}, Y|_{[0,t]})
\end{equation}
The combination of  \cref{eqn:sig_integral,eq:U_der} yields the relation
{\small
\begin{align*}
    U_\gamma(s,t) &=\frac{\partial}{\partial \epsilon} k\Big((X+\epsilon \gamma)|_{[0,s]}, Y|_{[0,t]}\Big)\Big|_{\epsilon = 0}  \\
    &= \frac{\partial}{\partial \epsilon}\left(\int_0^s\int_0^t U(u,v) \left(\dot{X}_u + \epsilon \dot{\gamma}_u\right)^T\dot{Y}_vdudv\right)_{\epsilon = 0}  \\
    &=  \int_0^s\int_0^t\left(U_\gamma(u,v)\dot{X}_u^T\dot{Y}_v + U(u,v) \dot{\gamma}_u^T\dot{Y}_v\right)dudv
\end{align*}
}%
Hence, differentiating the last equation first with respect to $t$ and then $s$ we get that the directional derivative $k_\gamma$ of the signature kernel along the path $\gamma$ solves the following PDE
\begin{equation}\label{eqn:deriv_PDE}
    \frac{\partial^2 U_\gamma}{\partial s \partial t} = (\dot{X}_s^T\dot{Y}_t)U_\gamma+  (\dot{\gamma}_s^T\dot{Y}_t)U
\end{equation}
with boundary conditions 
\begin{align}
    U_\gamma(0,\cdot) = 0, \ \ \ U_\gamma(\cdot, 0) =0.
\end{align} 
As a result, the gradients in \cref{eqn:meaning} of the signature kernel with respect to each of its input paths can be computed in a single call to a PDE solver, which concatenates the original state and the partial derivatives (\ref{eqn:deriv_PDE}) into a single vector. Each partial derivative follows the dynamics of (\ref{eqn:deriv_PDE}) where one replaces the direction $\gamma$ by the relevant localised impulse $\gamma_{i,j}$, $\tau_{i,j}$ for $X$ and $Y$ respectively. We outline the resulting procedure in \Cref{alg:PDE1}, where the concatenated partial derivatives are denoted by $U_{\boldsymbol{\gamma}}(s,t)$. Note that to optimize the ELBO we only need to differentiate $k(X,X)$, which is the case presented in the algorithm. The generalization to the case $k(X,Y)$ is straightforward using the chain rule.

\subsection{An explicit solution by variation of parameters}

From this second PDE (\ref{eqn:deriv_PDE}) we derive the following theorem (proved in \Cref{sec:additional_proofs}), that allows to compute the directional derivative $k_\gamma$ of the signature kernel directly from its evaluations at $X,Y$ and at $\overleftarrow{X},\overleftarrow{Y}$, where $\overleftarrow{X},\overleftarrow{Y}$ are respectively the paths $X,Y$ reversed in time.

\begin{theorem}\label{thm:kernelvarparams}
For any $\gamma \in \mathcal{X}$ the directional derivative $k_\gamma(X,Y)$ of the signature kernel along the path $\gamma$ satisfies the following relation 
\begin{align*}
     k_\gamma(X,Y) = \int_0^{T}\int_0^{T}  U(s,t) \widetilde U(T-s,T-t) (\dot{\gamma}_s^T\dot{Y}_t)dsdt
\end{align*}
where $\widetilde U(s,t) = k(\overleftarrow{X}|_{[0,s]},\overleftarrow{Y}|_{[0,t]})$ and where $\overleftarrow{X},\overleftarrow{Y}$ are respectively the paths $X,Y$ reversed in time.
\end{theorem}
The full backpropagation procedure is described in \Cref{alg:PDE2}.

\section{Related work}

In this section we expand on the material presented in \Cref{sec:Background}, focusing on the most recent approaches to scalable GPs on  $\mathbb{R}^d$ with VOFs and on sparse GPs for sequential data.

\paragraph{Variational Fourier Features} In \citet{hensman2017variational} the inducing variables are defined for scalar input $\mathcal{X}=\mathbb{R}$ as projections of the GP-sample onto the truncated Fourier basis. This type of inducing variables can be constructed for GPs with Mat\'{e}rn-type kernels. Although the resulting covariance matrix of the inducing variables is not diagonal, it can be decomposed into the sum of a diagonal matrix and rank one matrices. As a result it can be inverted using the \emph{Woodbury identity}, which makes it possible to scale GP inference on $\mathbb{R}$. The generalization to GPs on $\mathbb{R}^d$ is done by taking the outer product of the Fourier basis on $\mathbb{R}$. 
\paragraph{Eigenfunction inducing features} Closest to our work is the eigenfunction inducing features developed by \citet{burt2020convergence}, where the inducing variables are also defined as projections of the GP-sample onto an orthogonal basis of functions for the RKHS associated with the GP kernel. This relies on a \emph{Mercer's expansion} of the kernel. From here one identifies this orthogonal basis functions by solving an eigendecomposition problem. For example
\citet{dutordoir2020sparse} map the input data to the hypersphere $S^{d-1}\subset\mathbb{R}^d$ and then show that \textit{spherical harmonics} form an orthogonal basis for RKHS associated to \emph{zonal kernels} defined on $S^{d-1}$. 
\paragraph{GPs with signature covariances}
\citet{toth2020bayesian} propose a different sparse GP inference framework for sequential data with signature covariances (GPSig). In this work the inducing variables are either taken to be \emph{inducing sequences} (IS) in the original input space (GPSig-IS) of sequences or \emph{inducing tensors} (IT) in the corresponding feature space (GPSig-IT). The chosen covariance function is an approximation of the signature kernel based on truncating the signature to a finite level $n$. For GPSig-IT, this truncation makes the feature space finite dimensional and allows to optimize inducing tensors defined over such truncated space. Unlike our method, the inducing tensors are additional variational parameters to optimize. The covariance matrix $C_{\mathbf{uu}}$ is dense and its inversion incurs an additional $\mathcal{O}(M^3)$ cost. In Table \ref{table:Complexities} we compare the computational complexities of GPSig-IT, GPSig-IS and SigGPDE. A similar table for the memory complexity can be found in \Cref{sec:algorithmic_details}.

\begin{table}[h]
\begin{center}
\resizebox{\linewidth}{!}{%
\begin{tabular}{l c c c}
 \toprule
  Operation &  SigGPDE (ours) & GPSig-IT & GPSig-IS \\
    \midrule
    $C_{\mathbf{u}\mathbf{u}}$ & $\mathcal{O}(1)$ & $\mathcal{O}(n^2M^2d)$ & $\mathcal{O}((n+d)M^2\tilde{\ell}^2)$\\
    $C_{\mathbf{f}\mathbf{u}}$ & $\mathcal{O}(\tilde{N}M\ell)$&  $\mathcal{O}(n^2\tilde{N}M\ell d)$ & $\mathcal{O}((n+d)\tilde{N}M\tilde{\ell}\ell)$\\
    $\mathrm{diag}(C_{\mathbf{ff}})$ & $\mathcal{O}(d\tilde{N}\ell^2)$ &  $\mathcal{O}((n+d)\tilde{N}\ell^2)$ & $\mathcal{O}((n+d)\tilde{N}\ell^2)$ \\
    Lin. Alg. &$\mathcal{O}(\tilde{N}M^2)$  &  $\mathcal{O}(\tilde{N}M^2+M^3)$ & $\mathcal{O}(\tilde{N}M^2+M^3)$\\
    \bottomrule
\end{tabular}%
}
\caption{Comparison of time complexities. $M$ is the number of inducing variables, $\tilde{N}$ the batch size, $d$ the number of channels in the time series, $\ell$ the length of the sequences, $n$  the truncation level (for GPSig-IT and GPSig-IS) and $\tilde{l}$ the length of the inducing sequences. The last line of the table corresponds to linear algebra operations including matrix multiplication and matrix inversion.
}
\label{table:Complexities}
\end{center}
\vspace{-10pt}
\end{table}

\begin{figure*}
    \centering
    \includegraphics[trim={7cm 0cm 5cm 0},clip,scale=0.4]{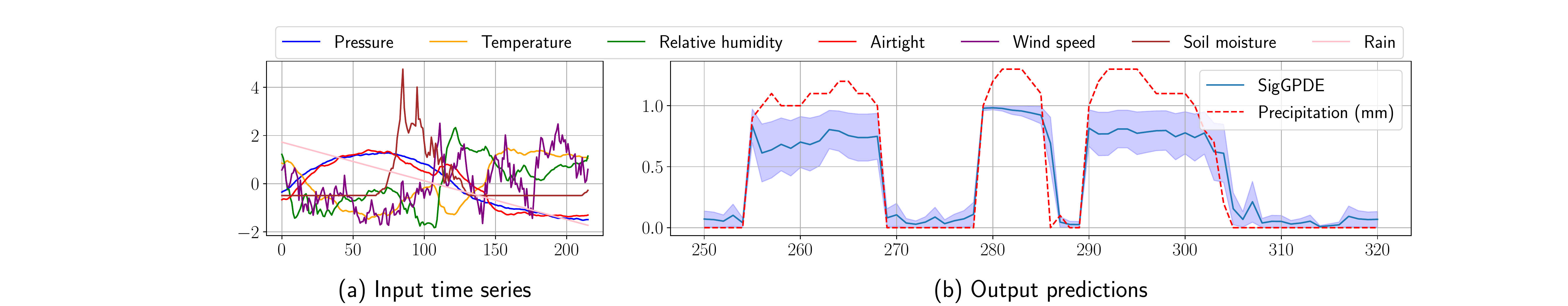}
  \caption{Weather forecast dataset. (a) One (standard scaled) multivariate time series $\mathbf{x}$ in input to the GP model. (b) Posterior mean of the SigGPDE GP when evaluated at multiple input time series like $\mathbf{x}$ on the test set. The actual precipitation amount is given for reference.}
    \label{fig:rainfall}
\end{figure*}

\section{Experiments}\label{sec:experiments}
In this last section, we benchmark SigGPDE against GPSig-IT and GPSig-IS from \citet{toth2020bayesian} on various multivariate time series classification tasks. For GPSig-IS, we use inducing sequences of length $\tilde{\ell}=5$ as recommended in \citet{toth2020bayesian}. We highlight how SigGPDE performs competitively in terms of accuracy and uncertainty quantification but with a significant speed-up in the fitting compared to the other baselines. 


We use a mixture of UEA \& UCR time series datasets (\url{timeseriesclassification.com}) and real world data for the final example. In the latter we discuss how the predictions provided by SigGPDE can be interpreted in a natural way via the interpretability of the interated integrals defining the signature and discussed in \Cref{ssec:sig-features}. 

We measure the classification accuracy on the test set, assess the uncertainty quantification with mean negative log-predictive probabilities (NLPP) and report the runtime per-iteration. For each dataset all models are trained $3$ times using a random training-validation split. The validation split is used to monitor the NLPP when optimizing the hyperparameters of the models. Further details on the training procedure can be found in \Cref{sec:experimental_details}. All code is written in TensorFlow using GPFlow \cite{de2017gpflow}.

\subsection{Classifying digits in sequential MNIST}\label{ssec:pendigits}
We start with a handwritten digit classification task, where writers were asked to draw the digits from $0$ to $9$. The instances are made up of $2$-d trajectories of the pen traced across a digital screen. The trajectories are of length $\ell=8$. The training and test sets are of size $7\,494$ and $3\,498$ respectively. We made use of $M=500$ inducing features. In the results reported in Table \ref{table:pendigits}, SigGPDE achieves even better accuracy and NLPP than the GPSig baselines, whilst being almost twice as fast than GPSig-IT.

\begin{table}[h]
\caption{Classification for sequential MNIST (PenDigits). The higher the Mean Acc. and the lower the NLPP the better.}
\begin{center}
\resizebox{\linewidth}{!}{%
\begin{tabular}{lccc}
 \toprule
Model & Mean Acc. & NLPP & Time \\
\midrule 
GPSig-IS & $97.42\pm 0.17$ & $0.096\pm0.005$ &  $0.186$ (s/iter)\\
GPSig-IT & $96.66\pm 0.59$ & $0.115\pm0.018$ & $0.036$ (s/iter)\\
\midrule

SigGPDE & $97.73\pm 0.13$& $0.085\pm0.001$ & $0.022$ (s/iter)\\ 
\bottomrule
\end{tabular}%
}
\label{table:pendigits}
\end{center}
\end{table}
\subsection{Detecting whale call signals}\label{ssec:whales}
In this example the task is to classify audio signals and distinguish one emitted from right whales from noise. The dataset (called RightWhaleCalls in the UEA archive) contains $10\,934$ train cases and $5\,885$ test cases. The signals are one-dimensional, sampled at $2$kHz over $2$ seconds, hence of length $4\,000$. We tackle this problem as a multivariate time series classification task, by taking the spectrogram of the univariate audio signal. The resulting streams are made of $29$ channels corresponding to selected frequencies and are $30$ time steps long.  The results in Table \ref{table:whales} are obtained with $M=700$ and show the significant speed-up of SigGPDE by almost one order of magnitude compared to GPSig. This speed-up is compensated by a minimal decrease in performance both in terms of accuracy and NLPP. 

\begin{table}[h]
\caption{Classification for whale call signals}
\begin{center}
\resizebox{\linewidth}{!}{%
\begin{tabular}{lccc}
 \toprule
Model & Mean Acc. & NLPP & Time \\
\midrule 
GPSig-IS & $86.97\pm 0.11$ & $0.367\pm0.005$ & $0.438$ (s/iter)\\
GPSig-IT  & $87.70\pm 0.42$ & $0.357\pm0.003$ & $0.048$ (s/iter) \\
\midrule
SigGPDE & $86.76\pm 0.36$ & $0.382\pm0.002$ & $0.008$ (s/iter)\\ 
\bottomrule
\end{tabular}%
}
\label{table:whales}
\end{center}
\end{table}

\subsection{Large scale classification of satellite time series}\label{ssec:crops}
This is our large scale classification example on 
$1$ million time series. The time series in this dataset represent a vegetation index, calculated from remote sensing spectral data. The $24$ classes represent different land cover types \cite{petitjean2012satellite}. The aim in classifying these
time series of length $\ell=46$ is to map different vegetation profiles to different types of crops and
forested areas. Due to the sheer size of this dataset we only compare SigGPDE to GPSig-IT as GPSig-IS is not scalable to such large dataset. In \Cref{fig:crops} we report the accuracy, time per iteration and ELBO by progressively increasing the number of inducing variables. Compared to SigGPDE, GPSig-IT has additional variational parameters, namely the inducing tensors. This extra flexibility explains the better performances of GPSig-IT when few inducing variables are used. However, as the number of inducing features increases, SigGPDE catches up and outperforms its competitor in all monitored metrics.

\begin{figure}[h]
    \centering
    \includegraphics[trim={0.3cm 0cm 0.cm 0},clip,scale=0.37]{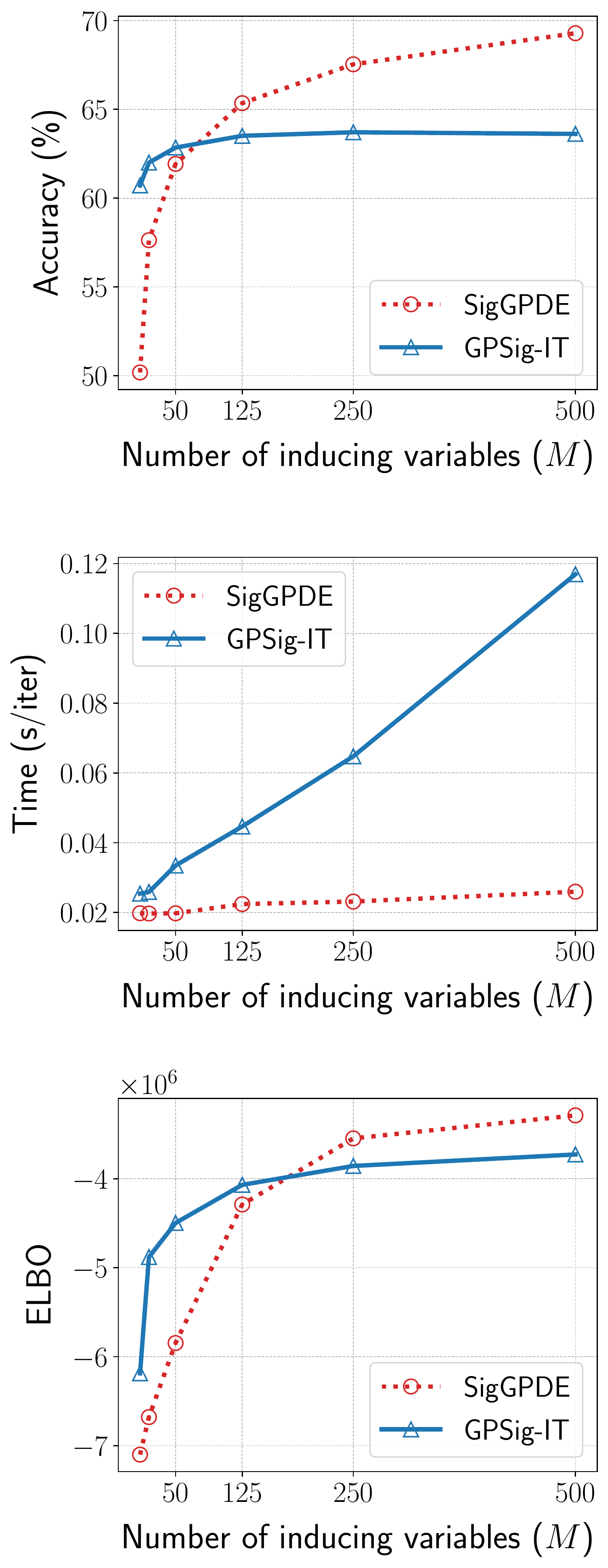}
    \caption{Large scale ($1$M) classification of satellite time series. Comparison of various metrics as functions of inducing variables.}
    \label{fig:crops}
\end{figure}

\subsection{Weather forecast}\label{ssec:weather}

In this last example we will be using a dataset of climatic variables recorded by the Max Planck Institute for Biogeochemistry\footnote{\url{https://www.bgc-jena.mpg.de/wetter/}} in the weather stations of WS Beutenberg and WS Saaleaue from 2004-2020. The data consists of $7$-dimensional time series recorded once per $10$ minutes where each channel represents a weather feature such as temperature, pressure, humidity etc. The goal is to predict whether it will rain over the next hour from the trajectory of all other features in the preceding $6$ hours. To obtain binary labels for the classification task we set the label to $1$ if the precipitation is larger than $1$mm and to $0$ otherwise. The inference mechanism is depicted on \Cref{fig:rainfall}.

A key feature proper to our model SigPDE is its interpretability. Looking at the variational mean vector $\mathbf{m}$ in \cref{eq:approximate_posterior}, we can extract the terms with highest relevance learned by the model. As discussed in \Cref{ssec:sig-features}, thanks to the corresponding signature features it is possible to infer which signature features used by the GP are more responsible for the produced outcome. The most relevant predictive features for this weather forecast experiment are represented in \Cref{fig:interpretability}. 

\begin{figure}[H]
    \centering
    \includegraphics[scale=0.4]{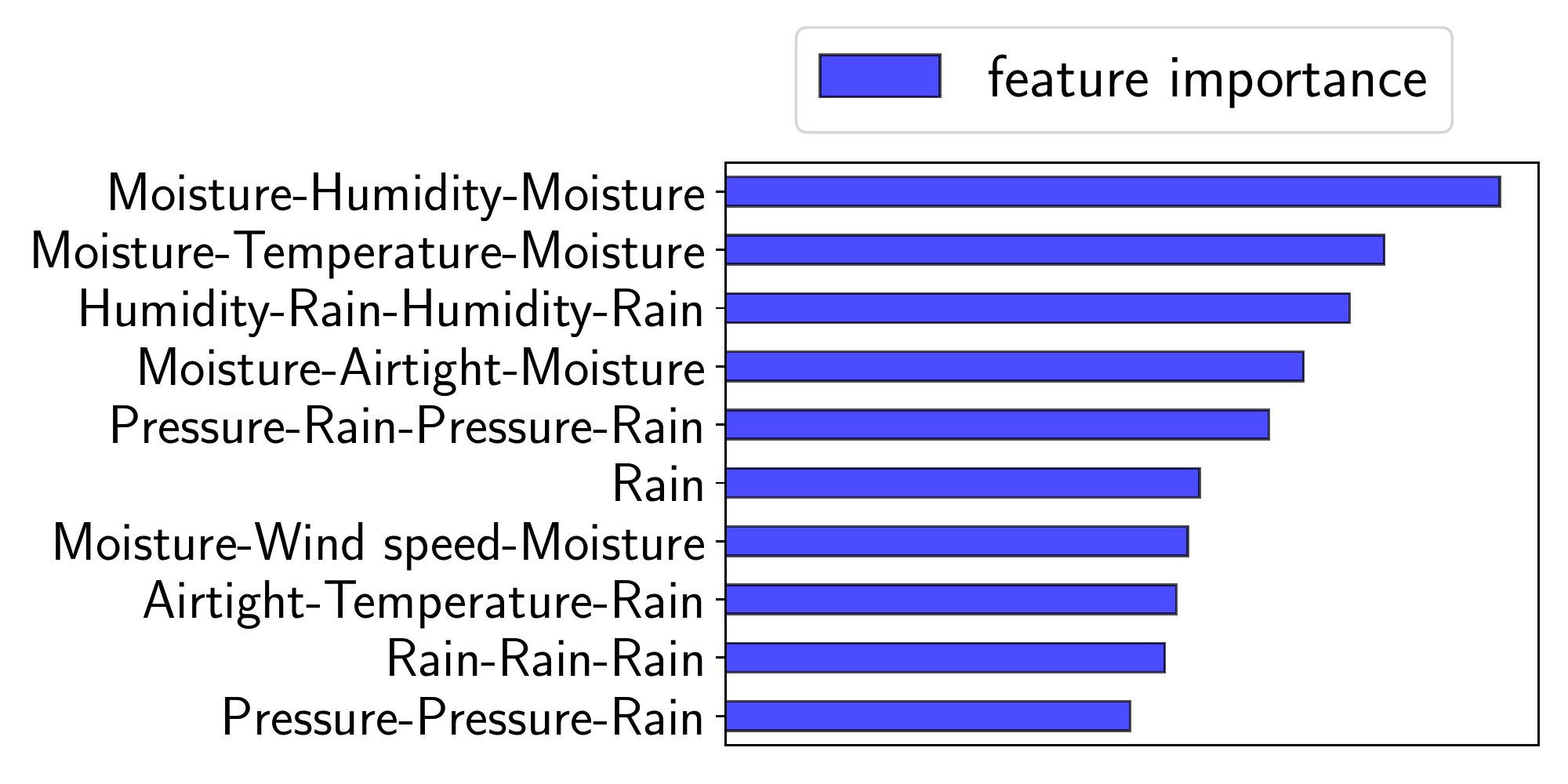}
    \caption{Top $10$ signature features (by importance) used by SigGPDE to predict whether or not it will rain in the next hour from previous weather data. Each feature is a term in the signature. For example \textit{Moisture-Humidity-Moisture} means that a change in the moisture channel followed by a change in the humidity channel and a change in the moisture channel is an important pattern.}
    \label{fig:interpretability}
\end{figure}

\section{Conclusion}
    In this paper we have developed SigGPDE, a framework to perform variational inference for GP models on sequential data with orthogonal signature features. Firstly, we constructed inducing variables so that their covariance matrix is diagonal. Secondly, we showed that the gradients of the signature kernel are solutions of a hyperbolic PDE. As a result the ELBO is cheap to evaluate as gradient descent does not require backpropagating through the operations of the PDE solver. We benchmarked SigGPDE against the state-of-the-art GPSig on different time series classification tasks, showing a significant speed up and similar performance.

\section*{Acknowledgements}
ML and CS were respectively supported by the EPSRC grants EP/L016710/1 and EP/R513295/1. TD acknowledges support from EPSRC (EP/T004134/1), UKRI Turing AI Fellowship (EP/V02678X/1), and the Lloyd’s Register Foundation programme on Data Centric Engineering through the London Air Quality project. ML, CS, TC and TL were supported by the Alan Turing Institute under the EPSRC grant EP/N510129/1 and DataSig under the grant EP/S026347/1.

\bibliography{references}
\bibliographystyle{icml2021}


\newpage

\appendix 

\twocolumn[
\icmltitle{Supplementary for\\ SigGPDE: Scaling Sparse Gaussian Processes on Sequential Data}
]

\section{Additional Proof}\label{sec:additional_proofs}

In this section we prove \Cref{thm:kernelvarparams} from the  main paper which yields an efficient algorithm to compute the gradients of the signature kernel with respect to its input paths. We recall \Cref{thm:kernelvarparams} below.

\begin{customthm}{4.1}
For any $\gamma \in \mathcal{X}$ the directional derivative $k_\gamma(X,Y)$ of the signature kernel along the path $\gamma$ satisfies the following relation 
\begin{align*}
     k_\gamma(X,Y) = \int_0^{T}\int_0^{T}  U(s,t) \widetilde U(T-s,T-t) (\dot{\gamma}_s^T\dot{Y}_t)dsdt
\end{align*}
where $\widetilde U(s,t) = k(\overleftarrow{X}|_{[0,s]},\overleftarrow{Y}|_{[0,t]})$ and where $\overleftarrow{X},\overleftarrow{Y}$ are respectively the paths $X,Y$ reversed in time.
\end{customthm}

Before proving \Cref{thm:kernelvarparams} we need the following important lemma.

\begin{lemma}\label{lemma:ker_inv}
For any two paths continuous paths of bounded variation $X,Y \in \mathcal{X}$ the signature kernel satisfies the following relation
\begin{equation}\label{eqn:inv}
    k(X,Y) = k(\overleftarrow{X},\overleftarrow{Y})
\end{equation}
where $\overleftarrow{X},\overleftarrow{Y}$ are the respectively $X,Y$ reversed in time.
\end{lemma}

\begin{proof}
It is a standard result in rough path theory (see for example \citep{lyons2007differential}) that $S(\overleftarrow{X}) = S(X)^{-1}$, where the inverse is taken in the set of grouplike elements, which is a group. The operator on grouplike elements $g : S(X) \mapsto S(X)^{-1}$ reverses the order of the letters in each word and multiplies the result by $-1$ if the length of the word is odd. Expanding out $k(\overleftarrow{X},\overleftarrow{Y})$ coordinate-wise it is easy to see that the two $-1$'s for words of odd length cancel as multiplied together, therefore the expansion of $k(X, Y)$ matches the one of $k(\overleftarrow{X},\overleftarrow{Y})$.
\end{proof}

Recall the notation for the signature kernel and its directional derivative used in the statement of  \Cref{thm:kernelvarparams}:
\begin{align*}
    U(s,t) &:= k\left(X|_{[0,s]}, Y|_{[0,t]}\right)\\
    U_\gamma(s,t) &:= k_\gamma\left(X|_{[0,s]}, Y|_{[0,t]}\right)
\end{align*}

\begin{proof}[Proof of Theorem \ref{thm:kernelvarparams}]

Let $\gamma : [0,T] \to \mathbb{R}^d$ be a continuous path of bounded variation along which we wish to differentiate $k$. Let's assume that for any $s,t \in [0,T]$ there exists a function $A_{s,t}:[0,T] \times [0,T] \to \mathbb{R}$ such that
\begin{equation}\label{eq:var_par}
    U_\gamma(s,t)= \int_0^s\int_0^t A_{s,t}(u,v) U(u,v) (\dot{\gamma}_u^T\dot{Y}_v)dudv
\end{equation}
Differentiating $U_\gamma$ with respect to $s$ and $t$ we get
\begin{align}
    \frac{\partial^2 U_\gamma}{\partial s \partial t} &= \int_0^s \int_0^t \frac{\partial^2
    A_{s,t}(u,v)}{\partial s \partial t}U(u,v) (\dot{\gamma}_u^T\dot{Y}_v)dudv \nonumber\\
    &+  A_{s,t}(s,t)U(s,t)(\dot{\gamma}_s^T \dot{Y}_t)\label{eqn:b}
\end{align}
By \cref{eqn:deriv_PDE} in the main paper we know that the directional derivative of the signature kernel along the path $\gamma$ solves the following PDE
\begin{equation}\label{eqn:a}
    \frac{\partial^2 U_\gamma}{\partial s \partial t} = (\dot{X}_s^T\dot{Y}_t)U_\gamma(s,t) +  (\dot{\gamma}_s^T\dot{Y}_t)U(s,t) 
\end{equation}
Equating \cref{eqn:b} and \cref{eqn:a} we deduce that $A_{s,t}(s,t)=1$ and 
\begin{align*}
    &\int_0^s \int_0^t \frac{\partial^2 A_{s,t}(u,v)}{\partial s \partial t}U(u,v) (\dot{\gamma}_u^T\dot{Y}_v)dudv \\
    &= U_\gamma(s,t)(\dot{X}_s^T\dot{Y}_t)\\
    &= (\dot{X}_s^T\dot{Y}_t)\int_0^s\int_0^t A_{s,t}(u,v)U(u,v) (\dot{\gamma}_u^T\dot{Y}_v) dudv
\end{align*}
Which implies that
\begin{equation}
    \frac{\partial^2 A_{s,t}(u,v)}{\partial s \partial t} = (\dot{X}_s^T\dot{Y}_t)A_{s,t}(u,v)
\end{equation}
Or equivalently, by integrating with respect to $s$ and $t$
\begin{equation}
    A_{s,t}(u,v) = 1 + \int_{u}^s\int_{v}^tA_{s',t'}(u,v) (\dot{X}_{s'}^T\dot{Y}_{t'}) ds'dt'
\end{equation}
Hence
\begin{align}
    A_{T,T}(u,v) &= \langle S(X)_{[u,T]}, S(Y)_{[v,T]} \rangle \\
    &= k(\overleftarrow{X}|_{[0,T-u]},\overleftarrow{Y}|_{[0,T-v]})
\end{align}
where the last equality is a consequence of \Cref{lemma:ker_inv}. Pluging back this result into \cref{eq:var_par} concludes the proof.
\end{proof}


\section{Additional Experimental Details}\label{sec:experimental_details}
In this section we describe the experimental setup for \Cref{sec:experiments}. All experiments were conducted on NVIDIA Tesla P100 GPUs.

    \subsection{Data collection process}\label{ssec:data_collection}
    
    The classification tasks of \Cref{ssec:pendigits,ssec:whales} were performed on two datasets (PenDigits, RightWhaleCalls) from the UCR \& UEA time series classification
repository.\footnote{\url{https://timeseriesclassification.com}} For the large scale classification experiment of \Cref{ssec:crops} we used a dataset of 1M satellite time series (STS).\footnote{\url{https://cloudstor.aarnet.edu.au/plus/index.php/s/pRLVtQyNhxDdCoM}} Lastly, the climatic data (WeatherForecast) for rainfall prediction task in \Cref{ssec:weather} was downloaded from the Max Planck Institute for Biogeochemistry website.\footnote{\url{https://www.bgc-jena.mpg.de/wetter/weather_data.html}} 

Data pre-processing included the following two steps. As explained in \Cref{ssec:sig-features}, we first add a monotonically increasing coordinate to all multivariate time series that we call "time", which effectively augments by one the number of channels. This is a standard procedure employed within signature based methods \citep{toth2020bayesian,chevyrev2016primer}. Then, we standard scale the time series using tslearn library \citep{JMLR:v21:20-091}. This is particularly important for the WeatherForecast dataset where channels have different scales. Additional processing steps have been performed for two datasets (RightWhaleCalls, WeatherForecast) which we treat separately next. 

A standard data transformation to tackle classification tasks on audio signals consists in computing their \textit{spectrograms}. We follow this procedure for the RightWhaleCalls dataset which contains univariate highly-oscillatory time series of length $2\,000$. We used the scipy Python library to do so. The spectrogram is commonly represented as a graph with one axis representing time, the other axis representing the frequency, and the color intensity representing the amplitude of a particular frequency at a particular time. In this paper, we consider the spectrogram as a multivariate time series, where each channel represents the change in amplitude of a particular frequency over time. Furthermore, exploiting the fact that frequencies in whale call signals are typically between $50$ and $300$Hz, we only consider frequencies which fall within this range. As a result we obtain $28$-dimensional time series each of length $30$. We then apply the pre-processing steps described above.

To create the WeatherForecast dataset we used the recordings of various climatic variables in two weather stations located in Germany from $2004$ to $2020$. The outliers were filtered out, and we used the recordings of $7$ variables (depicted on \Cref{fig:rainfall}) over $6$ hours in order to predict whether it would rain by more than $1$mm over the next hour. There is one recording every $10$min resulting in input time series of length $\ell=36$. Since there were much fewer positive cases (raining) than negative cases (not raining), we dropped at random a fraction of the data, such that the ratio of positive/negative examples is brought down to $3$.

    \subsection{Training procedure}\label{ssec:training_procedure}
    
      The datasets for classification of sequential digits (PenDigits), audio signals (RightWhaleCalls), and satellite time series (STS) come with a predefined test-train split. In order to report standard deviations on our results we subsampled $20\%$ (PenDigits,RightWhaleCalls) or $2\%$ (STS) of the training set to form a validation set. 
    
     The training was equally split into $3$ different phases. During the first phase, only the variational parameters are trained. For the second phase, both the variational parameters and the hyperparameters of the kernel are trained. During the last phase the variational parameters are trained on the full training set (the validation data being merged back). Overall, the hyperparameters are fixed for two-third of the iterations. SigGPDE and the GPSig-IT/IS baselines have the same set of hyperparameters, which correspond to the scaling factors for each channel for the ARD parametrization of the signature kernel \Cref{ssec:parametrization}. Those were initialized with the same value for all models. The inducing tensors for GPSig-IT and inducing sequences for GPSig-IS were initialized following the procedure outlined in \cite{toth2020bayesian}. We recall that for SigGPDE there is no such parameters to initialize. As recommended in \cite{toth2020bayesian}, we use a truncation level of $n=4$ for their signature kernel algorithm (GPSig-IT/IS).

    The minibatch size is either $50$ (PenDigits, RightWhaleCalls) or $200$ (STS).  We used the Nadam optimizer \citep{dozat2016incorporating} with learning rate $10^{-3}$. In the main paper we report the time per iteration which corresponds to one minibatch.

\section{Additional Algorithmic Details}\label{sec:algorithmic_details}

In this section we start by outlining the space and time complexities of the algorithms underlying SigGPDE. Then, we explain how we have developed a dedicated CUDA TensorFlow operator for GPU acceleration to speed-up the computation of the signature kernel and its gradients.

    \subsection{Complexity analysis}\label{ssec:complexity_analysis}
    
    The main algorithms underpinning SigGPDE consist in computing three different covariance matrices to evaluate the ELBO. These are the covariance matrix between the inducing variables $\mathbf{u}$ (denoted by $C_{\mathbf{uu}}$), between the marginal $\mathbf{f}$ and the inducing variables (denoted by $C_{\mathbf{fu}}$), and finally the covariance matrix of $\mathbf{f}$ (its diagonal is denoted by $\mathrm{diag}(C_{\mathbf{ff}})$). In \Cref{table:time,table:space} we compare the time and space complexities for the corresponding SigGPDE algorithms to those of GPSig-IT/IS. 
    
    In the SigGPDE sparse variational inference framework, $C_{\mathbf{uu}}$ is diagonal which lowers both the memory and computational costs (see first line \Cref{table:time,table:space}). Besides there is no need to compute the Cholesky decomposition of $C_{\mathbf{uu}}$ to invert it (see last line \Cref{table:time}). Lastly, in SigGPDE the inducing variables do not depend on any variational parameter (see last line \Cref{table:space}). 
\begin{table}[h]
\begin{center}
\resizebox{\linewidth}{!}{%
\begin{tabular}{l c c c}
 \toprule
  Operation &  SigGPDE (ours) & GPSig-IT & GPSig-IS \\
    \midrule
    $C_{\mathbf{u}\mathbf{u}}$ & $\mathcal{O}(1)$ & $\mathcal{O}(n^2M^2d)$ & $\mathcal{O}((n+d)M^2\tilde{\ell}^2)$\\
    $C_{\mathbf{f}\mathbf{u}}$ & $\mathcal{O}(\tilde{N}M\ell)$&  $\mathcal{O}(n^2\tilde{N}M\ell d)$ & $\mathcal{O}((n+d)\tilde{N}M\tilde{\ell}\ell)$\\
    $\mathrm{diag}(C_{\mathbf{ff}})$ & $\mathcal{O}(d\tilde{N}\ell^2)$ &  $\mathcal{O}((n+d)\tilde{N}\ell^2)$ & $\mathcal{O}((n+d)\tilde{N}\ell^2)$ \\
    Lin. Alg. &$\mathcal{O}(\tilde{N}M^2)$  &  $\mathcal{O}(\tilde{N}M^2+M^3)$ & $\mathcal{O}(\tilde{N}M^2+M^3)$\\
    \bottomrule
\end{tabular}%
}
\caption{Comparison of time complexities. $M$ is the number of inducing variables, $\tilde{N}$ the batch size, $d$ the number of channels in the time series, $\ell$ the length of the sequences, $n$  the truncation level (for GPSig-IT and GPSig-IS) and $\tilde{\ell}$ the length of the inducing sequences.
}
\label{table:time}
\end{center}
\vspace{-10pt}
\end{table}

\begin{table}[h]
\begin{center}
\resizebox{\linewidth}{!}{%
\begin{tabular}{l c c c}
 \toprule
  Operation &  SigGPDE (ours) & GPSig-IT & GPSig-IS \\
    \midrule
    $C_{\mathbf{u}\mathbf{u}}$ & N/A & $\mathcal{O}(n^2M^2)$ & $\mathcal{O}(M^2\tilde{\ell}^2)$\\
    $C_{\mathbf{f}\mathbf{u}}$ & $\mathcal{O}(\tilde{N}M\ell)$& $\mathcal{O}(n^2\tilde{N}M\ell)$  &$\mathcal{O}(\tilde{N}M\ell\tilde{\ell})$ \\
    $\mathrm{diag}(C_{\mathbf{ff}})$ &  $\mathcal{O}(\tilde{N}\ell^2)$& $\mathcal{O}(\tilde{N}\ell^2)$  &  $\mathcal{O}(\tilde{N}\ell^2)$\\
    \midrule
    $\mathbf{z}$ & N/A & $\mathcal{O}(n^2Md)$ & $\mathcal{O}(M\tilde{\ell}d)$\\
    \bottomrule 
\end{tabular}%
}
\caption{Comparison of space complexities, separated by algorithm to compute each covariance matrix. The last line accounts for the storage of the inducing tensors and inducing sequences in GPSig-IT and GPSig-IS.
}
\label{table:space}
\end{center}
\vspace{-10pt}
\end{table}

    \subsection{Computing the signature kernel and its gradients}\label{ssec:computing_sig_kernel}
    
    Recall that the signature kernel solves the following PDE,
\begin{align}
    \frac{\partial^2 U}{\partial s \partial t} = (\dot{X}_s^T\dot{Y}_t)U && U(0,\cdot)=1,~U(\cdot,0)=1
\end{align}
therefore each kernel evaluation amounts to a call to a PDE solver. Using a straightforward implementation of a finite-difference PDE solver which consists in applying an update of the form 
\begin{align*}
    U(s_i,t_j)=g(U(s_{i-1},t_{j-1}),U(s_{i},t_{j-1}),U(s_{i-1},t_{j})),
\end{align*}
in row or column order, the time complexity for $N$ kernel evaluations for time series with $d$ channels of length $\ell$ is $\mathcal{O}(dN\ell^2)$. Indeed there is no data dependencies between each of the $N$ kernel evaluations, hence we can solve each PDE in parallel. But, this does not reduce the quadratic complexity with respect to the length $\ell$.  However, it is possible to parallelize the PDE solver by observing that instead of solving the PDE in row or column order, we can update the antidiagonals of the solution grid. As illustrated on \Cref{fig:cuda}, each cell on an antidiagonal can be updated with in parallel as there is no data dependency between them. Therefore, we propose a CUDA implementation where $N$ collections of $2\ell-1$ threads (the number of cells on the biggest antidiagonal) running in parallel can simultaneously update an antidiagonal of the solution grids. 

To compute the gradients, we use the result from \Cref{thm:kernelvarparams}. During the forward pass we solve the PDEs defined by the input time series using the CUDA operator described above. For the backward pass, we first solve the PDEs with the input time series reversed in time, by calling the same CUDA operator. Second, we compute the gradients using simple vectorized TensorFlow operations. 

\begin{figure}
    \centering
    \includegraphics[scale=1.2]{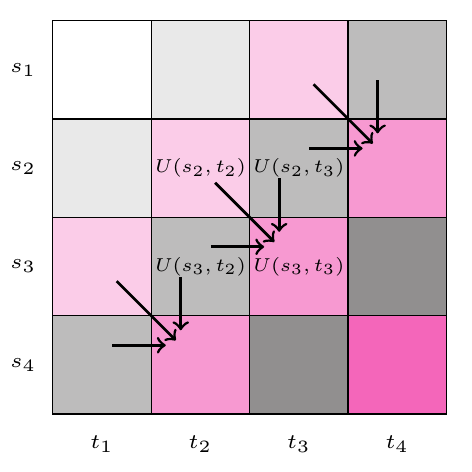}
    \caption{Parallelization of the finite-difference scheme. Each cell on an antidiagonal can be computed in parallel, provided the previous antidiagonals have been computed.}
    \label{fig:cuda}
\end{figure}
    


\end{document}